\documentclass[10pt]{article} 
\usepackage[preprint]{tmlr}

\usepackage{amsmath,amsfonts,bm}

\def\eqref#1{equation~\ref{#1}}

\def\1{\bm{1}}

\DeclareMathAlphabet{\mathsfit}{\encodingdefault}{\sfdefault}{m}{sl}
\SetMathAlphabet{\mathsfit}{bold}{\encodingdefault}{\sfdefault}{bx}{n}

\newcommand{\E}{\mathbb{E}}

\newcommand{\R}{\mathbb{R}}

\usepackage{hyperref}
\usepackage{url}
\usepackage{amsmath,amssymb,amsthm,mathtools}
\usepackage{algorithm}
\usepackage{algpseudocode}
\usepackage{booktabs}
\usepackage[capitalise,noabbrev]{cleveref}
\usepackage{bbm}

\title{Private Sparse Sketches for Linear Regression}

\author{\name Shrutimoy Das \email shrutimoydas@iitgn.ac.in \\
      \addr Department of Computer Science and Engineering \\
      Indian Institute of Technology Gandhinagar
      \AND
      \name Anirban Dasgupta \email anirbandg@iitgn.ac.in\\
      \addr Department of Computer Science and Engineering \\
      Indian Institute of Technology Gandhinagar\\
      }

\def\month{MM}  
\def\year{YYYY} 
\def\openreview{\url{https://openreview.net/forum?id=XXXX}} 

\newcommand{\OPT}{\mathrm{OPT}}
\newcommand{\cmax}{c_{\max}}
\newcommand{\bmo}{\beta_{-1}}
\newcommand{\Fpub}{\mathcal{F}_{\rm pub}}
\newcommand{\N}{\mathcal{N}}
 
\newtheorem{theorem}{Theorem}[section]
\newtheorem{lemma}[theorem]{Lemma}
\newtheorem{proposition}[theorem]{Proposition}
\newtheorem{corollary}[theorem]{Corollary}
\theoremstyle{definition}
\newtheorem{definition}[theorem]{Definition}
\theoremstyle{remark}

\crefname{algorithm}{Algorithm}{Algorithms}

\begin{document}

\maketitle

\begin{abstract}
Sketching compresses a large dataset into a smaller summary, allowing regression on the sketch to approximate the full-data objective. A sketch released under differential privacy can be reused with any non-private regression solver, any number of times, without further privacy loss. We study private sparse sketches for \(\ell_2\) and \(\ell_1\) regression. We first prove lower bounds on the noise required by common sparse sketches for both problems when the sketching matrix is public. Since regression on a sketch can be solved without revealing the sketching matrix, we construct sketches whose randomness is kept secret. Hiding how individual records enter the sketch provides an additional source of privacy, allowing less noise for the same privacy guarantee. We prove differential privacy guarantees for both constructions, and experiments on real datasets show improved utility over public-sketch baselines.
\end{abstract}

\section{Introduction}\label{sec:intro}
Linear regression is widely used with sensitive data, including modeling clinical outcomes in healthcare~\citep{Raffa2016}, predicting healthcare expenditures~\citep{Yang2018}, and estimating corporate credit risk from financial information~\citep{LAITINEN199997}. Differentially private (DP) linear regression aims to estimate the solution to a regression problem while protecting the contribution of individual data records. Given a data matrix $X\in\mathbb{R}^{n\times d}$ and a response vector $y\in\mathbb{R}^n$, the $\ell_p$ regression problem is
\[
    \beta^* = \arg\min_{\tilde{\beta}\in\mathbb{R}^d}\|X\tilde{\beta}-y\|_p^p,
\]
where $p=2$ corresponds to least-squares regression ($\ell_2$ regression) and $p=1$ to least absolute deviation regression ($\ell_1$ regression). A large body of work studies differentially private $\ell_2$ regression, including methods based on perturbing sufficient statistics such as $X^\top X$ and $X^\top y$~\citep{wang2018revisitingdifferentiallyprivatelinear}, random projections~\citep{She17}, private stochastic optimization~\citep{varshney2022sgd}, and exponential-mechanism approaches~\citep{liu2021differentialprivacyrobuststatistics}. More recent work considers stable estimators obtained by perturbing insufficient statistics on a suitably selected subset of the data~\citep{brown2024insufficientstatisticsperturbationstable}. For $\ell_1$ regression, \citet{wang2022differentiallyprivateell1normlinear} use the exponential mechanism, while \citet{liu2024frappe} develop an iterative noise-injection procedure for sparse least absolute deviation regression.
 
In this work, we consider a different approach: rather than privately releasing the regression solution itself, we release a \emph{private sketch of the data} from which the regression problem can subsequently be solved. This approach is based on the classical \emph{sketch-and-solve} paradigm in numerical linear algebra~\citep{sketch2014woodruff}. A sketching matrix $S$ maps the original data matrix to a substantially smaller matrix $SX$, while preserving the geometry relevant to the optimization problem. Appropriate sketching constructions are known to provide low distortion embeddings for both $\ell_2$ and $\ell_1$ regression~\citep{rp2006sarlos,NN13,MM13,CW17,subspace2011sohler,mestimator2015clarkson,MOW23}. Consequently, one can solve the regression problem on the sketch using an ordinary non-private regression solver.
 
Releasing the sketch itself has an important privacy advantage. Once a sketch has been released with differential privacy, any number of regression problems or other computations performed solely on the released sketch incur no additional privacy loss, by the postprocessing property of differential privacy. Thus, the sketch acts as a reusable private summary of the data: the data owner performs the privacy-preserving computation once, after which the released object can be used with standard non-private optimization procedures. This perspective is related to the idea of private coresets, which combines data compression with differential privacy to produce a private representation that can support subsequent queries~\citep{private2009feldman,private2017feldman}. In this work, our goal is to design private sketches specifically for preserving the objectives of $\ell_1$ and $\ell_2$ regression.
 
The sketches best suited to this goal are sparse. CountSketch~\citep{CCF04,CW17} is an oblivious subspace embedding for $\ell_2$ regression that can be applied in $O(\mathrm{nnz}(X))$ time, and the multi-level sketch of \citet{MOW23} gives an $O(1)$-approximation for $\ell_1$ regression with similar cost. This paper studies the question: \emph{how should sparse sketches be designed under differential privacy constraints?} The direct approach draws $S$ independently of the data and adds Gaussian noise to the sketch calibrated to its sensitivity. Sparsity makes this expensive. In a CountSketch, a single row enters one bucket with its full norm, so every bucket needs noise of the size required to hide a whole record. In the $\ell_1$ sketch, the deep levels hold few rows and carry weights as large as order $n$, and these weights multiply the noise.
 
We first show that this cost is unavoidable when $S$ is public. For every noise distribution, a private CountSketch with $r$ buckets must add noise energy of order $rB^2/\varepsilon^2$, where $B$ bounds the row norms, and this acts as a ridge penalty of the same size. For the $\ell_1$ sketch, the noise adds to the sketched objective a constant fraction of $nB\|\beta_{-1}\|_2$, the largest value the data term can take. These lower bounds assume that the observer knows $S$. The sketch-and-solve paradigm never needs $S$, because the regression is solved on the released sketch alone. We therefore keep the hash functions, random signs and sampling decisions secret and discard them once the sketch is computed. To an observer, each row then lands in an unknown bucket with an unknown sign, and this uncertainty is an additional source of privacy. For the resulting Hidden CountSketch, hiding the bucket divides the $\chi^2$-divergence contributed by a row by $r$, and hiding the sign makes it fourth order rather than second order in $B/\sigma$. We compute the moments of the likelihood ratio exactly, which yields a R\'enyi accountant and a certified noise level under replace-one adjacency, and we prove utility bounds from the subspace-embedding property of the sketch. We further extend the same approach to $\ell_1$ regression by hiding the sampling and hashing of the multi-level sketch. To the best of our knowledge, this is the first differentially private sketch construction specifically for $\ell_1$ regression.
 
The resulting private summaries can be processed using standard non-private $\ell_1$ or $\ell_2$ regression solvers and can be reused without further privacy loss. We complement our theoretical analysis with experiments across different privacy budgets. For $\ell_2$ regression, hiding the randomness reduces the implicit ridge parameter by a factor of $44$ to $587$ relative to the public CountSketch, and Hidden CountSketch is the most accurate private sketch at every finite $\varepsilon$, with the advantage most pronounced when $\varepsilon$ is small. For $\ell_1$ regression, hiding the randomness reduces the noise level by a factor of $6$ to $14$, but the large importance weights of the sketch still dominate the error, as our lower bound predicts for public sketches.
 
\section{Related Work}\label{sec:related}
\paragraph{Sketching for regression.}
Sparse subspace embeddings give input-sparsity-time algorithms for $\ell_2$ regression~\citep{CW17,NN13,MM13}, and several constructions exist for $\ell_p$ and robust regression~\citep{subspace2011sohler,mestimator2015clarkson,MOW23}; see \citet{sketch2014woodruff} for a survey. Their analyses use only that $S$ is independent of the data and never require $S$ to be revealed. We keep these sketches unchanged and modify only what is released and how much noise is added, so their approximation guarantees carry over.
 
\paragraph{Private $\ell_2$ regression.}
Most private regression methods output an estimator for one problem~\citep{wang2018revisitingdifferentiallyprivatelinear,Abadi_2016,varshney2022sgd,liu2021differentialprivacyrobuststatistics}, whereas we release a sketch that any solver can reuse. Releasing $X^\top X$ and $X^\top y$ is also reusable, but serves only least squares. The closest work is the private Johnson--Lindenstrauss (JL) sketch of \citet{She17}, which builds on the result of \citet{BBDS12} that a secret Gaussian projection is itself private when $\sigma_{\min}(A)$ is large. This requires a private test of $\sigma_{\min}(A)$, and when the test fails the data are padded, which adds a ridge penalty. Our hidden-sign argument also uses the symmetry of secret randomness, but it holds for every dataset with no condition on $\sigma_{\min}(A)$, and the sketch is sparse.
 
\paragraph{Private $\ell_1$ regression.}
Existing methods~\citep{wang2022differentiallyprivateell1normlinear,liu2024frappe} and DP-SGD release a single estimator. We release a private $\ell_1$ sketch, and our lower bound identifies what limits this approach: the importance weights of the sketch scale the noise, whatever its distribution.
 
\paragraph{Private sketches and coresets.}
Private sketches have been developed for low-rank approximation and streaming linear algebra~\citep{upadhyay2017differentiallyprivatelinearalgebra}, distributed data analysis~\citep{burkhardt2025distributeddifferentiallyprivatedata}, kernel sums~\citep{onepass2021coleman}, and frequency and quantile estimation~\citep{dpsketch2022zhao}. Some sketches are private because of their internal randomness, such as the JL transform~\citep{BBDS12} and the Flajolet--Martin sketch, which needs no added noise~\citep{SST20}. We show that the hidden randomness of sparse regression sketches reduces the noise needed but does not remove it. Our guarantee holds against anyone who sees only the sketch; unlike pan-private algorithms~\citep{DNPRY10}, it does not survive disclosure of the sketching randomness. Private coresets~\citep{private2009feldman,private2017feldman} are weighted subsets built from data-dependent sampling distributions, which require additional care when the dataset changes; our sketches are linear and data-oblivious.
 
\paragraph{Privacy amplification.}
Hiding which records are used where underlies privacy amplification by subsampling~\citep{KLNRS08,BBG18}, shuffling~\citep{EFMRTT19}, random allocation~\citep{FS25}, and iteration~\citep{FMTT18}. The secret hash of CountSketch is an instance of random allocation, since each row is assigned to one of $r$ buckets uniformly at random. Our analysis differs in three ways. First, the random sign has no counterpart in amplification by routing: it makes each row's contribution symmetric and cancels the first-order term of the privacy loss. With $t=B/\sigma$, hiding only the bucket gives a $\chi^2$-divergence of $(e^{t^2}-1)/r$, while hiding the sign as well gives $(\cosh(t^2)-1)/r$. Second, the buckets are released together as one Gaussian mixture rather than as outputs of separate steps; we compute its likelihood-ratio moments exactly and bound the divergence under replace-one adjacency directly. Third, in the $\ell_1$ sketch each row takes part in every level, with geometrically decreasing sampling rates, so we analyse each level as a subsampled and allocated mixture and compose across levels. Our lower bounds show that without such hiding, no noise distribution closes the gap.

\section{Preliminaries}
We consider a dataset consisting of a feature matrix $X\in\mathbb{R}^{n\times d}$ and a response vector $y\in\mathbb{R}^n$, where $n>d$ and assume that $X$ has full column rank. Let $A=[X ; y]\in\mathbb{R}^{n\times(d+1)}$ be the matrix obtained by appending $y$ to the columns of $X.$ Following \citep{She17,She19}, we assume that each row of $A$ has $\ell_2$ norm at most $B.$ 
   For $\beta\in\mathbb{R}^d$, let $\beta_{-1}=[\beta^\top, -1]^\top\in\mathbb{R}^{d+1}$, so that $A\beta_{-1}=X\beta-y$. We define the optimal $\ell_2$ and $\ell_1$ regression losses as
\[
\OPT=\min_{\beta\in\mathbb{R}^d}\|X\beta-y\|_2^2,
\qquad
\OPT_1=\min_{\beta\in\mathbb{R}^d}\|X\beta-y\|_1.
\]
Let $A=U\Sigma V^\top$ be a singular value decomposition of $A$, and denote its smallest singular value as $\sigma_{\min}(A).$  

Two datasets, $D, D'$ are \emph{neighbouring} if they differ in a single row. We say that $D$ and $D'$ are replace-one adjacent.

\begin{definition}[Approximate Differential Privacy ~\citep{DR14}]
 A randomized mechanism $\mathcal{M}$ satisfies $(\varepsilon,\delta)$-differential privacy (DP) if, for every pair of neighbouring datasets $D,D'$ and every measurable event $\mathcal E$,
\[
\Pr[\mathcal{M}(D)\in \mathcal E]
\le e^\varepsilon\Pr[\mathcal{M}(D')\in \mathcal E]+\delta.
\]   
\end{definition}
 
The $\ell_2$ sensitivity for a matrix-valued function is
\[
\Delta_2(f)
=\sup_{D\sim D'}
\big\|\operatorname{vec}\big(f(D)-f(D')\big)\big\|_2,
\]
where $D\sim D'$ denotes neighbouring datasets and $\operatorname{vec}(\cdot)$ denotes vectorization. 

We use $\mu,\nu\in(0,1)$ to denote the approximation error and failure probability, respectively. We denote the indicator of an event by $\mathbbm{1}\{\cdot\}$ and the absolute value of a scalar $z$ by $|z|$. Let $\langle u,v \rangle$ denote the inner product of the vectors $u$ and $v.$ The notation $\mathcal{N}(0,\sigma^2 I_d)$ denotes the $d$-dimensional isotropic Gaussian distribution with mean zero and covariance matrix $\sigma^2 I_d$. 

\begin{definition}[Oblivious Subspace Embedding ~\citep{CW17}]
    A distribution over matrices $S\in\mathbb{R}^{r\times n}$ is a $(\mu,\nu)$ oblivious subspace embedding (OSE) if, for every fixed matrix $M\in\mathbb{R}^{n\times d'}$, the following holds with probability at least $1-\nu$:
\[
(1-\mu)\|Mv\|_2
\le \|SMv\|_2
\le (1+\mu)\|Mv\|_2
\qquad
\text{for all }v\in\mathbb{R}^{d'}.
\]
Here, the probability is over the randomness of $S.$
\end{definition}

The \emph{sketch-and-solve} paradigm for linear regression ~\citep{sketch2014woodruff} consists of computing a sketch $SX \in \R^{r \times d}$ of the feature matrix $X \in \R^{n \times d}$ using a $r \times n$ sketching matrix S, and then solving the sketched regression problem $\min_{\bar \beta} \Vert SX\bar\beta - Sy \Vert_q^q,$ where $q \in \{1,2\}.$

\section{Private Sparse Sketch}\label{sec:pub}

A straightforward way for computing a private sketch is to sample the sketching matrix $S$ independent of $A,$ compute the sketch $SA,$ and privatize the sketch as $\tilde A = SA + E,$ where $E$ is a $r \times (d + 1)$ matrix where $E_{ij}$ is the noise calibrated to the $\ell_2$ sensitivity of the sketch. For $\ell_2$ regression, we consider Countsketch ~\cite{CCF04, CW17} for computing the sketch whereas for $\ell_1$ regression we consider the multilevel sketch proposed in ~\cite{MOW23}.

\subsection{\texorpdfstring{Private CountSketch for $\ell_2$ regression}{Private CountSketch for L2 regression}}\label{sec:pub-l2}
 
A CountSketch matrix $S\in\{-1,0,1\}^{r\times n}$~\citep{CCF04,CW17} is specified by a hash function $h:[n]\to[r]$ and a sign vector $g\in\{\pm1\}^n$, sampled independently and uniformly at random. Its $i$th column is $S_{:,i}=g_i e_{h(i)}$, where $e_j$ denotes the $j$th standard basis vector. We refer to each row of $SA$ as a \emph{bucket}, containing the signed sum of the rows of $A$ assigned to that bucket. CountSketch gives an oblivious subspace embedding with $r=O((d+1)^2/(\mu^2\nu))$~\citep{CW17,NN13,MM13}, and the sketch can be computed in $O(\mathrm{nnz}(A))$ time.
 
\begin{lemma}[Sensitivity of $SA$]\label{lem:sens}
For every fixed CountSketch matrix $S$, the $\ell_2$ sensitivity of the mapping $A\mapsto SA$, is $2B.$
\end{lemma}
\begin{proof}
Let $A$ and $A'$ be neighboring datasets that differ only in row $i$. Since the $i$th column of $S$ is $g_i e_{h(i)}$, we have
\(
S(A-A')=g_i e_{h(i)}(A_{i,:}-A'_{i,:}).
\)
Consequently,
\[
\| SA- SA'\|_F
=\|A_{i,:}-A'_{i,:}\|_2
\le \|A_{i,:}\|_2+\|A'_{i,:}\|_2
\le 2B.
\]
\end{proof}

The pseudo-code computing a private Countsketch of $A$ is given in Algorithm ~\ref{alg:cs}. In Line ~\ref{alg_line:sensitivity}, the variance of $\mathcal N(0, \sigma^2)$ is either calibrated with respect to the $\ell_2$ sensitivity of $A \mapsto SA$ (~\cref{lem:sens}), or computed as in the Analytic Gaussian mechanism (AGM) ~\citep{BW18}.
 
\begin{algorithm}[t]
\caption{Private CountSketch}\label{alg:cs}
\begin{algorithmic}[1]
\Require $A\in\R^{n\times(d+1)}$ with row norms at most $B$, sketch size $r$, privacy parameters $(\varepsilon,\delta)$.
\Ensure Private sketch $\tilde A \in\R^{r\times(d+1)}$.
\State Sample a CountSketch matrix $S\in\{-1,0,1\}^{r\times n}$ independently of $A$.
\State Sample $E\in\R^{r\times(d+1)}$ with $E_{jk}\stackrel{\rm iid}{\sim}\N(0,\sigma^2).$ i.i.d \label{alg_line:sensitivity} 
\State Release $\tilde A=SA+E$.
\end{algorithmic}
\end{algorithm}
 
\begin{lemma}[Privacy of \cref{alg:cs}]\label{lem:priv}
\cref{alg:cs} is $(\varepsilon,\delta)$-DP. 
\end{lemma}
\begin{proof}

\cref{alg:cs} is $(\varepsilon,\delta)$-DP by the Gaussian mechanism ~\citep{DR14}.
\end{proof}
 
\paragraph{Price of privacy.}

Adding noise to $SA$ changes the $\ell_2$ regression problem. We can interpret this change by augmenting $A$ with a regularization term. Define
\begin{equation}\label{eq:T}
\hat A=\begin{bmatrix}A\\ \sigma\sqrt r\, I_{d+1}\end{bmatrix},
\qquad
T=\left[\,S\ \ \frac{1}{\sigma\sqrt r}E\,\right].
\end{equation}
Then $T\hat A=SA+E=\tilde A$. Setting $\lambda=r\sigma^2$, we have
\[
\|\hat A\bmo\|_2^2
=\|X\beta-y\|_2^2+\lambda\bigl(\|\beta\|_2^2+1\bigr).
\]
The added term therefore corresponds to ridge regularization with the ridge parameter $\lambda$. The private sketch can thus be viewed as a sketch of an augmented $\ell_2$ regression problem with regularization parameter $\lambda$.

A similar observation was noted in~\citep{She17} for the Johnson--Lindenstrauss (JL) sketch, where the corresponding ridge parameter is $\lambda_{\rm JL}=w^2$. This lets us compare the two sketches in terms of the regularization they introduce. For $\varepsilon\le 1$, the Gaussian noise scale satisfies $\sigma=\Theta(B\sqrt{\ln(1/\delta)}/\varepsilon)$, giving
\[
\lambda_{\rm CS}=r\sigma^2
=\Theta\Big(\frac{rB^2\ln(1/\delta)}{\varepsilon^2}\Big),
\qquad
\lambda_{\rm JL}
=\Theta\Big(\frac{B^2\big(\sqrt{r\ln(1/\delta)}+\ln(1/\delta)\big)}{\varepsilon}\Big).
\]
When $r > \ln(1/\delta)$ and the two sketches have the same size, the implicit ridge parameter for CountSketch is larger by a factor of order $\sqrt{r\ln(1/\delta)}/\varepsilon$.  
 
\subsection{\texorpdfstring{Private multi-level sketch for $\ell_1$ regression}{Private multi-level sketch for L1 regression}}\label{sec:pub-l1}
 
For $\ell_1$ regression, we privatize the oblivious sketch of \citet{MOW23}, which achieves an $O(1)$-approximation. The sketching matrix, $S,$ is formed by stacking levels $S_0,\dots,S_{h_m}$. At level $0$, each row of $A$ is mapped to $s\ge 1$ buckets. At each level $h\ge1$, rows are independently sampled with probability $p_h\propto b^{-h},$ for $b ^{h_m} = n,$ and assigned to one of $N_h$ buckets, with each bucket weighted by $1/p_h$. The final level uses a uniform subsample with $N_u$ buckets. Rows assigned to the same bucket are summed. Thus, $S$ has entries in $\{0,1\}$, and the weight vector $w$ is independent of the data.
 
\begin{algorithm}[t]
\caption{Private multi-level sketch for $\ell_1$ regression}\label{alg:l1}
\begin{algorithmic}[1]
\Require $A\in\R^{n\times(d+1)}$ with row norms at most $B$, sketch parameters $(r,b,s,N_u)$, privacy parameters $(\varepsilon,\delta)$.
\Ensure Private weighted sketch $(\tilde A,w)$.
\State Construct $S\in\{0,1\}^{r\times n}$ and $w\in\R^r_{>0}$ as in \citet{MOW23}, independently of $A$.
\State Compute $\cmax=\max_{i\in[n]}\|S_{:,i}\|_2.$ 
\State Sample $E\in\R^{r\times(d+1)}$ with $E_{jk}\stackrel{\rm iid}{\sim}\N(0,\sigma^2)$. \label{alg_line:sens}
\State \Return $(\tilde A,w)$ with $\tilde A = SA+E$.
\end{algorithmic}
\end{algorithm}

\begin{lemma}[Privacy of \cref{alg:l1}]\label{lem:priv1}
The $\ell_2$ sensitivity of the mapping $A\mapsto SA$ is $2B\cmax$ with $\cmax\le\sqrt{s+h_m}$, and \cref{alg:l1} is
$(\varepsilon,\delta)$-DP.
\end{lemma}

\begin{proof}
Let $A$ and $A'$ be neighboring datasets that differ only in row $i$. Then
\[
\|SA-SA'\|_F
=\|S_{:,i}\|_2\|A_{i,:}-A'_{i,:}\|_2
\le 2B\|S_{:,i}\|_2.
\]
Each column of $S$ has $s$ entries equal to $1$ at level $0$ and at most $1$ at each of the remaining $h_m$ levels. Hence, $\|S_{:,i}\|_2\le\sqrt{s+h_m}.$ 

The Gaussian mechanism with noise calibrated to this sensitivity ensures ~\cref{alg:l1} is $(\varepsilon,\delta)$-DP. Since $w$ is data-independent, releasing it does not affect the privacy guarantee.
\end{proof}

The pseudo-code computing the private sketch of $A$ is given in  ~\cref{alg:l1}. In Line ~\ref{alg_line:sens}, the variance of $\mathcal N(0, \sigma^2)$ is either calibrated with respect to the $\ell_2$ sensitivity of $A \mapsto SA$ (~\cref{lem:priv1}), or computed as in the Analytic Gaussian mechanism (AGM) ~\citep{BW18}.

\subsection{Lower Bounds}\label{sec:lb}

The privacy guarantees of \cref{alg:cs} and \cref{alg:l1} hold even if the sketching matrix, or the seed used to
generate it, is made public. We refer to this setting as the \emph{public-randomness} model and formalise it as follows.
\begin{definition}[Public-randomness sketches]\label{def:fpub}
Let $\Fpub$ denote the family of mechanisms that release $\tilde A=SA+E$, together with any data-independent quantities
such as the weights $w$, where $S\in\R^{r\times n}$ is sampled independently of $A$ and $E\in\R^{r\times(d+1)}$ is sampled
independently of $A$ conditional on $S$. We require that, for every fixed realisation of $S$, the mechanism
$A\mapsto SA+E$ is $(\varepsilon,\delta)$-DP.
\end{definition}
The noise distribution in this definition can be any distribution and its entries need not be independent. Also, the noise  may depend
on $S$, as it does in \cref{alg:l1} through $\cmax$. The condition of \cref{def:fpub} implies that releasing $S$ together
with $\tilde A$ is $(\varepsilon,\delta)$-DP. 
Both constructions of \cref{sec:pub} belong to $\Fpub$.

We use this model to establish lower bounds on the privacy cost of sparse sketches for both $\ell_2$ and $\ell_1$
regression. The bounds apply to every mechanism in $\Fpub$ , whatever the noise distribution. Our proofs combine a two-point argument ~\citep{LeCam1973ConvergenceOE} based on group
privacy, in the spirit of the packing lower bounds of \citet{HT10} and \citet{Vad17}, with the hypothesis-testing
interpretation of DP~\citep{WZ10,KOV15}.
 
\begin{lemma}[Group privacy ~\citep{Vad17}]\label{lem:group}
If  a mechanism $\mathcal M$ is $(\varepsilon,\delta)$-DP and $A,A''$ differ in $k$ rows, then for every event $\mathcal T$,
$\Pr[\mathcal M(A)\in\mathcal T]\le e^{k\varepsilon}\Pr[\mathcal M(A'')\in\mathcal T]+\delta_k$ with
$\delta_k=\delta\sum_{i=0}^{k-1}e^{i\varepsilon}\le ke^{(k-1)\varepsilon}\delta$.
\end{lemma}

\begin{lemma}\label{lem:shift}
Let $\eta$ be a real random variable and $z >0$. Suppose that $\eta$ and $\eta+z$ are
$(\varepsilon',\delta')$-indistinguishable in both directions, that is, for every measurable $\mathcal T\subseteq\R$,
\[
\Pr[\eta+z\in\mathcal T]\le e^{\varepsilon'}\Pr[\eta\in\mathcal T]+\delta'
\quad\text{and}\quad
\Pr[\eta\in\mathcal T]\le e^{\varepsilon'}\Pr[\eta+z\in\mathcal T]+\delta'.
\]
Then
\[
\Pr\big[|\eta|\ge z/2\big]\ \ge\ \frac{2(1-\delta')}{e^{\varepsilon'}+1},
\qquad\text{and consequently}\qquad
\E\eta^2\ \ge\ \frac{z^2}{2}\cdot\frac{1-\delta'}{e^{\varepsilon'}+1}.
\]
The hypothesis is unchanged if $\eta$ is replaced by $\eta-c$ for a constant $c$, so the same bounds hold for
$\eta-c$.
\end{lemma}

\Cref{lem:shift} gives a lower bound on the spread of $\eta$ needed to hide a shift of size $z$. If $|\eta|<z/2$ with high probability, then thresholding at $z/2$ would reliably distinguish $\eta$ from $\eta+z$, violating the indistinguishability condition. Thus, $|\eta|\ge z/2$ with probability bounded away from zero.
The detailed proof is given in \cref{app:lb}.
 
\paragraph{The $\ell_2$ lower bound.}
For a matrix $S\in\R^{r\times n}$, an output row $j$ and an integer $k$, let $s_j(k)=\max_{|G|=k}\sum_{i\in G}|S_{ji}|$ be
the largest total weight that $k$ input rows, denoted by $G,$  carry into output row $j$.
 
\begin{theorem}[Lower bound for public sketches, $\ell_2$]\label{thm:lb2}
Let a mechanism in $\Fpub$ release $\tilde A =SA+E$, let $S$ be fixed, and let $E$ have finite second moments. For every unit
vector $u\in\R^{d+1}$, every output row $j$ and every $k\ge1$,
\[
\E\big[(E_{j,:}u)^2\big]\ \ge\ \frac{2B^2s_j(k)^2(1-\delta_k)}{e^{k\varepsilon}+1},
\qquad\text{hence}\qquad
\E\big[E^\top E\big]\succeq\frac{2B^2(1-\delta_k)}{e^{k\varepsilon}+1}\Big(\sum_{j=1}^r s_j(k)^2\Big)I_{d+1}.
\]
\end{theorem}
 
\begin{corollary}[Private CountSketch, ~\cref{alg:cs}]\label{cor:lbcs}
If $S$ is a CountSketch in which every bucket receives at least $k$ rows of $A$, then
\[
\E\big[E^\top E\big]\succeq\lambda_{\rm pub}I_{d+1},\qquad
\lambda_{\rm pub}\ \ge\ \frac{2rB^2k^2(1-\delta_k)}{e^{k\varepsilon}+1}.
\]
In particular, for $k=\lceil2/\varepsilon\rceil$, $\lambda_{\rm pub}\ge\dfrac{8rB^2(1-\delta_k)}{\varepsilon^2(e^{2+\varepsilon}+1)}$.
\end{corollary}



The detailed proof is given in ~\cref{app:lb}.
Since the total noise determines the implicit ridge regularization, the lower bound rules out improving the $\Theta(r/\varepsilon^2)$ rate of \cref{alg:cs} by changing the noise distribution.

\paragraph{The $\ell_1$ lower bound.}
For weights $w_j\ge0$, let $W = \operatorname{diag}(w)$ and $F_{\tilde A}(\beta)=\|W \tilde{A}\bmo\|_1=\sum_jw_j|\tilde A_{j,:}\bmo| = \sum_jw_j|SA_{j,:}\bmo + E\bmo|$ be the sketched $\ell_1$ objective.
 
\begin{theorem}[Lower bound for public sketches, $\ell_1$]\label{thm:lb1}
Let a mechanism in $\Fpub$ release $\tilde A=SA+E$, where $S$ is
fixed and $w\geq0$ is a vector of weights. Then, for every
$\beta\in\R^d$,
\[
\E_E\,F_{\tilde A}(\beta)
\geq
\frac{2(1-\delta)}{e^\varepsilon+1}
B\|\bmo\|_2
\sum_{j=1}^r w_j\|S_{j,:}\|_\infty.
\]
\end{theorem}
 
\begin{corollary}[Private multi-level sketch, ~\cref{alg:l1}]\label{cor:lbl1}
Consider any mechanism in $\Fpub$ built on the sketch of
Algorithm~\ref{alg:l1}, with an arbitrary noise distribution.
Let $W_{\neq\emptyset}$ denote the total weight of the non-empty
buckets. Then, for every $\beta\in\R^d$,
\[
\E_E F_{\tilde A}(\beta)
\geq
\frac{2(1-\delta)}{e^\varepsilon+1}
B\|\bmo\|_2\,W_{\neq\emptyset}.
\]
Furthermore, if the uniform level retains each input row with
probability $N_u/n$ and assigns weight $n/N_u$ to each bucket,
then
\[
\E_S W_{\neq\emptyset}\geq(1-e^{-1})n.
\]
Consequently,
\[
\E_{S,E}F_{\tilde A}(\beta)
\geq
\frac{2(1-e^{-1})(1-\delta)}{e^\varepsilon+1}
\,nB\|\bmo\|_2.
\]
\end{corollary}
 

Since $\|A_{i,:}\|_2\le B$, the data term satisfies $\|A\bmo\|_1\le nB\|\bmo\|_2$. ~\cref{cor:lbl1}
says that the noise alone contributes a constant fraction of this trivial upper bound to the sketched objective. As $h$ increases, the probability of sampling a point decreases since $p_h \propto b^{-h}.$ Hence, the deep levels of the sketch must resolve buckets that contain
very few rows, and give them large weights so that they represent the whole data set. The noise that protects a single row in the uniform sampling level is then multiplied by a weight of order $n$. The noise term behaves like a penalty on $\|\bmo\|_2$ that is
as large as the worst possible data fit, which pushes the minimiser of $F_{\tilde A}$ towards $\beta=0$. 

Both lower bounds rely on the  assumption that the observer knows $S$. Changing the noise
distribution, or calibrating it more tightly cannot escape this lower bound.  However, the \emph{sketch-and-solve} paradigm
never needs $S$, an observation that we exploit in ~\cref{sec:hcs}.
 
\section{Private Sparse Sketch via Hidden Randomness}\label{sec:hcs}
 
The lower bounds in \cref{sec:lb} assume that the sketching matrix $S$ is public. However, the \emph{sketch-and-solve} approach does not require $S$ to be known since regression can be solved using only the sketch, for both $\ell_2$ and $\ell_1$ losses. This means that the random choices used to construct an oblivious sketch, including row sampling, bucket assignments and signs, can be kept secret and discarded once the sketch has been computed.

In this section, we show that hiding these choices provides an additional source of privacy and can reduce the amount of noise needed. In \cref{sec:hcs-l2}, we apply this idea to CountSketch for $\ell_2$ regression; in \cref{sec:hcs-l1}, we consider the multilevel sketch for $\ell_1$ regression. 

\subsection{\texorpdfstring{$\ell_2$ Regression}{L2 Regression}}\label{sec:hcs-l2}

We first apply this idea to CountSketch. The hash function $h$ and sign vector $g$ are sampled independently of the data and discarded after the sketch is computed. An observer therefore does not know which bucket a row was assigned to or which sign it received. Hiding the bucket assignments is similar to the random allocation mechanism of \citet{FS25}, where each record is assigned uniformly at random to one of $r$ steps. This spreads the influence of each row across the buckets from the observer's perspective. Hidden signs provide an additional source of privacy: each row is equally likely to contribute $+A_{i,:}$ or $-A_{i,:}$ to its bucket. Together, these two forms of randomness allow us to use substantially less noise than in \cref{alg:cs}.

The resulting mechanism, which we call \emph{Hidden CountSketch}, is given in \cref{alg:hcs}. It differs from \cref{alg:cs} only in the noise level (line~\ref{alg_line:sigmastar}) and in discarding $h$ and $g$ after constructing the sketch (line~\ref{alg_line:discard}). Before discussing the computation of $\sigma$ in Line ~\ref{alg_line:sigmastar} of ~\cref{alg:hcs}, we derive its privacy and utility guarantees.
 
\begin{algorithm}[t]
\caption{Hidden CountSketch}\label{alg:hcs}
\begin{algorithmic}[1]
\Require $A\in\R^{n\times(d+1)}$ with row norms at most $B$, sketch size $r$, privacy parameters $(\varepsilon,\delta)$. 
\Ensure Private sketch $\tilde A \in\R^{r\times(d+1)}$.
\State Set $\sigma=B\,\sigma^\star(\varepsilon,\delta,r)$, where $\sigma^\star$ is the noise multiplier of \cref{eq:sigmastar}. \label{alg_line:sigmastar}
\State Draw $h(i)\sim\mathrm{Unif}[r]$ and $g_i\sim\mathrm{Unif}\{\pm1\}$ independently, using a cryptographically secure generator.
\State Release $\tilde A=SA+E$ with $E_{jk}\stackrel{\rm iid}{\sim}\N(0,\sigma^2)$, and discard $h$ and $g$. \label{alg_line:discard}
\end{algorithmic}
\end{algorithm}
 
\subsection{Privacy and Utility Guarantees}\label{sec:hcs-guarantees}
 
We analyse the privacy of ~\cref{alg:hcs} using the R\'enyi divergence of order $\alpha>1$, defined for distributions $P$ and $Q$ as
\[
D_\alpha(P\|Q)=\frac{1}{\alpha-1}\log \E_Q\left[\left(\frac{P}{Q}\right)^\alpha\right]
\]
\citep{Mir17,vEH14}. For integer $\alpha$, this reduces the analysis to moments of the likelihood  $P/Q$, which we can compute exactly for our mechanism. We then convert the resulting R\'enyi divergence bounds into an $(\varepsilon,\delta)$-DP guarantee . 
 
  We assume that the two datasets differ only in row $i,$ and  the observer knows the hash buckets and signs of all rows except row $i$. To analyse the privacy guarantee, we condition on the hashes and signs of the other rows which are independent of $(h(i),g_i)$ and have the same distribution for both datasets. Once conditioned on them, the contributions of all other rows are fixed and can be subtracted from $\tilde A$ without affecting the R\'enyi divergence. The remaining output then depends only on the hidden bucket and sign of row $i$. 

After subtracting the common contribution of the other rows, we consider three distributions:
\[
P_a=\frac{1}{2r}\sum_{j=1}^{r}\sum_{g\in\{-1,+1\}}
\mathcal{N}\big(g\,e_j a^\top,\sigma^2 I_{r(d+1)}\big),
\qquad
P_{a'} \text{ (defined analogously using $a'$)}, 
\qquad
R=\mathcal{N}\big(0,\sigma^2 I_{r(d+1)}\big).
\]
Here, $\mathcal{N}(S,\sigma^2 I_{r(d+1)})$ denotes the distribution of an $r\times(d+1)$ random matrix with independent entries, each having mean equal to the corresponding entry of $SA$ and variance $\sigma^2$. The distribution $R$ corresponds to replacing row $i$ with the zero vector. For each $j$, the matrix $e_j a^\top$ has $a^\top$ as its $j$th row and zeros in all other rows. Next, we derive the likelihood  ratio of $P_a$ relative to $R$ in ~\cref{lem:lr}.

\begin{lemma}[Likelihood ratio]\label{lem:lr}
Let $X\sim R$, with rows $X_1,\dots,X_r$. Then
\[
\frac{P_a}{R}(X)=\frac1r\sum_{j=1}^r\ell_a(X_j),
\qquad
\ell_a(x)=e^{-\|a\|_2^2/(2\sigma^2)}
\cosh\left(\frac{\langle a,x\rangle}{\sigma^2}\right).
\]
Moreover, $\ell_a(X_1),\dots,\ell_a(X_r)$ are i.i.d.\ copies of
\[
L=e^{-t_a^2/2}\cosh(t_a Z),
\]
where $Z\sim\mathcal{N}(0,1)$ and $t_a=\|a\|_2/\sigma$.
\end{lemma}

The sum averages over the hidden bucket assignment, while  $\cosh$ is the average of $e^{+x}$ and $e^{-x}$ over the hidden sign. \Cref{prop:chi2} quantifies the effect of hiding the bucket assignments and signs on the second moment of the likelihood ratio, corresponding to R\'enyi divergence of order $\alpha=2$.


Let
\(
M_m(t,r)=\E\left[\left(\frac1r\sum_{j=1}^r L_j\right)^m\right],
\)
where $L_1,\ldots,L_r$ are i.i.d.\ copies of $L$, and let $K(\gamma,t)=\E[L^{-\gamma}]$ for $\gamma>0$. The privacy bound in \cref{thm:hcs} depends on the moments $M_m(t,r)$ of the averaged likelihood ratio. \Cref{lem:moments} shows how to compute these moments exactly, allowing us to determine a noise level with a certified privacy guarantee.

 Under replace-one adjacency, we need to bound the divergence between $P_a$ and $P_{a'}$ directly. The following lemma relates this divergence to the moments of the likelihood ratio with respect to $R$, which are easier to compute.
\begin{lemma}\label{lem:decouple}
Let $\alpha>1$, let $m>\alpha$ be an integer, let $\gamma=\frac{m(\alpha-1)}{(m-\alpha)r}$, and let $t=B/\sigma$. For all
$a,a'$ with $\|a\|_2,\|a'\|_2\le B$,
\[
\E_{P_{a'}}\Big[\Big(\frac{P_a}{P_{a'}}\Big)^{\alpha}\Big]\ \le\ M_m(t,r)^{\alpha/m}\,K(\gamma,t)^{r(m-\alpha)/m}.
\]
\end{lemma}
 
  Although $P_a$ and $P_{a'}$ are each mixtures of $2r$ Gaussians, their likelihood ratio is difficult to handle directly. Their likelihood ratios with respect to $R$, by contrast, have simple forms (\cref{lem:lr}).

Combining \cref{lem:decouple} with the single-row analysis and the conversion from Rényi divergence to $(\varepsilon,\delta)$-DP gives the following guarantee for \cref{alg:hcs}.

\begin{theorem}[Privacy of \cref{alg:hcs}]\label{thm:hcs}
Let $t=B/\sigma$. For $\alpha>1$, define
\[
\rho(\alpha)=\min_{\substack{m\in\mathbb N\\m>\alpha}}\frac{1}{\alpha-1}\left[
\frac{\alpha}{m}\log M_m(t,r)
+\frac{r(m-\alpha)}{m}\log K\left(\frac{m(\alpha-1)}{(m-\alpha)r},t\right)
\right].
\]
Then the release in Algorithm ~\ref{alg:hcs}, with noise level $\sigma$, satisfies
\(
D_\alpha(\tilde A\,\|\,\tilde A')\le \rho(\alpha)
\)
for every pair of neighbouring datasets and every $\alpha>1$. In particular, it is $(\varepsilon,\delta)$-DP under replace-one adjacency, where
\[
\varepsilon=\min_{\alpha>1}\left\{
\rho(\alpha)+\log\frac{\alpha-1}{\alpha}
-\frac{\log\delta+\log\alpha}{\alpha-1}
\right\}.
\]
\end{theorem}
 
\begin{proof}[Proof sketch]
By \cref{lem:decouple}, $D_\alpha(P_a\|P_{a'})\le \rho(\alpha)$. Since the bound is symmetric in $a$ and $a'$, it holds in both directions. Joint quasi-convexity then gives the same bound for the unconditional distributions of $\tilde A$, as discussed above. Applying the Rényi-to-DP conversion of \citet{CKS20} yields the stated $(\varepsilon,\delta)$ guarantee. The details are in \cref{app:hcs}.
\end{proof}

The utility analysis of \cref{cor:hcsutil} uses only that $S$ is independent of $A$ and $E$, with $\lambda=r\sigma^2$ (~\cref{eq:T}) for the smaller noise level $\sigma=B\sigma^\star$.
 
\begin{theorem}[Utility of \cref{alg:hcs}]\label{cor:hcsutil}
Let $\beta=\arg\min_{\bar \beta\in\R^d}\|\tilde A \bar \bmo\|_2^2$ for the output of Algorithm ~\ref{alg:hcs}, and let $\lambda_{\textrm hid}=r\sigma^2$
with $\sigma=B\sigma^\star(\varepsilon,\delta,r)$. For $r\ge2(d+1)^2/(\mu^2\nu)$, with probability at least $1-\nu$,
\[
\|X\beta-y\|_2^2\le\frac{1+\mu}{1-\mu}\min_{\tilde\beta\in\R^d}\Big(\|X\tilde\beta-y\|_2^2+\lambda_{\textrm hid}\|\tilde\beta\|_2^2\Big)+\frac{2\mu}{1-\mu}\lambda_{\textrm hid}.
\]
\end{theorem}

Now, we discuss the calibration of the noise variance in Line \ref{alg_line:sigmastar} of \cref{alg:hcs}.
 
\paragraph{Calibration of the noise.}
Every quantity in Theorem ~\ref{thm:hcs} depends on $B$ and $\sigma$ only through their ratio $t=B/\sigma$. This is because scaling both the data and the noise by the same factor scales the output by that factor without changing the information it reveals. Let $\varepsilon_{\textrm HCS}(s;\delta,r)$ denote the value of $\varepsilon$ given by Theorem ~\ref{thm:hcs} when $t=1/s$. We choose the noise multiplier as
\begin{equation}\label{eq:sigmastar}
\sigma^\star(\varepsilon,\delta,r)=\inf\big\{s>0:\ \varepsilon_{\textrm HCS}(s;\delta,r)\le\varepsilon\big\}.
\end{equation}
This quantity depends only on the privacy parameters and the sketch size. 
For data with row norms bounded by $B$, we use noise with standard deviation $\sigma=B\sigma^\star$.
In our implementation, we minimise over a fixed grid of R\'enyi orders.  Each order gives a valid privacy bound.
We find $\sigma^\star$ by bisection on a logarithmic scale as shown in Lemma ~\ref{lem:moments}.

In our implementation, we minimise over a fixed grid of R\'enyi orders.  Each order gives a valid privacy bound.
We find $\sigma^\star$ by bisection on a logarithmic scale.



\subsection{\texorpdfstring{$\ell_1$ Regression}{L1 Regression}} \label{sec:hcs-l1}

We apply the same idea to the $\ell_1$ sketch of \citet{MOW23} used in \cref{alg:l1}. The sketch itself remains unchanged, but its randomness is kept secret. In this case, the hidden randomness determines which rows are sampled into each level and which buckets they are assigned to. Unlike CountSketch, the $\ell_1$ sketch uses no random signs. The resulting mechanism is given in \cref{alg:hmow}. It differs from \cref{alg:l1} only in the noise level (line~\ref{alg_line:sigmastar-mow}) and in discarding the sampling decisions and hash functions after constructing the sketch (line~\ref{alg_line:discard-mow}).
 
\begin{algorithm}[t]
\caption{Hidden multi-level sketch for $\ell_1$ regression}\label{alg:hmow}
\begin{algorithmic}[1]
\Require $A\in\R^{n\times(d+1)}$ with row norms at most $B$, sketch parameters of \citet{MOW23}, privacy parameters $(\varepsilon,\delta)$.
\Ensure Private weighted sketch $(\tilde A,w)$.
\State Set $\sigma=B\,\sigma^\star_{\rm MOW}(\varepsilon,\delta)$, where $\sigma^\star_{\rm MOW}$ is the noise multiplier of \cref{eq:sigmastar-mow}. \label{alg_line:sigmastar-mow}
\State Construct $S\in\{0,1\}^{r\times n}$ and $w$ as in \citet{MOW23}, independently, using a cryptographically secure generator.
\State Release $(\tilde A,w)$ with $\tilde A=SA+E$ and $E_{jk}\stackrel{\rm iid}{\sim}\N(0,\sigma^2)$, and discard the  hash functions and sampling decisions. \label{alg_line:discard-mow}
\end{algorithmic}
\end{algorithm}
 
\subsection{\texorpdfstring{Privacy and Utility Guarantees of the Hidden $\ell_1$ Sketch}{Privacy and Utility Guarantees of the Hidden L1 Sketch}}\label{sec:hmow-guarantees}
 
The analysis follows \cref{sec:hcs-guarantees}, with the differences begin that the $\ell_1$ sketch has no random signs, and each level samples a row with probability $p_{\kappa}$. As before, let neighbouring datasets differ only in row $i$, whose values are $a$ and $a'$, with $\|a\|_2,\|a'\|_2\le B$. We condition on the sampling decisions and hash assignments of all other rows across all levels, then subtract their contribution from the output. The remaining randomness comes from the sampling and bucket assignment of row $i$ in each level.

For a level with $N$ buckets and row-inclusion probability $p$, the resulting output has distribution
\[
P^{(N,p)}_a
=(1-p)R+\frac{p}{N}\sum_{j=1}^{N}
\mathcal{N}\big(e_j a^\top,\sigma^2I_{N(d+1)}\big),
\]
where $R=\mathcal{N}(0,\sigma^2I_{N(d+1)})$, as in \cref{sec:hcs-guarantees}. With probability $1-p$, row $i$ is not sampled into this level, so the output contains only noise. Otherwise, the row is assigned uniformly to one of the $N$ buckets. We derive the likelihood ratio for a level.
 
\begin{lemma}[Likelihood ratio of a level]\label{lem:lr-mow}
Let $X\sim R$, with rows $X_1,\dots,X_N$. Then
\[
\frac{P^{(N,p)}_a}{R}(X)
=(1-p)+\frac pN\sum_{j=1}^N\ell^\mathrm{ML}_a(X_j),
\qquad
\ell^\mathrm{ML}_a(x)
=\exp\left(\frac{\langle a,x\rangle}{\sigma^2}
-\frac{\|a\|_2^2}{2\sigma^2}\right).
\]
Moreover, $\ell^\mathrm{ML}_a(X_1),\dots,\ell^\mathrm{ML}_a(X_N)$ are i.i.d.\ copies of
\(
L^\mathrm{ML}=e^{t_aZ-t_a^2/2},
\)
where $Z\sim\mathcal{N}(0,1)$ and $t_a=\|a\|_2/\sigma$.
\end{lemma}

Without random signs, each bucket contributes $\ell^\mathrm{ML}_a$ to the likelihood ratio, rather than the $\cosh$ term in \cref{lem:lr}. 
 
The moments required for the privacy bound can be computed exactly. For $L^\mathrm{ML}$, which is lognormally distributed,
\(
\E[(L^\mathrm{ML})^c]=e^{c(c-1)t^2/2}
\)
for every integer $c\ge0$, while its negative moments are
\(
K^\mathrm{ML}(\gamma,t)
=\E[(L^\mathrm{ML})^{-\gamma}]
=e^{\gamma(\gamma+1)t^2/2},
\quad \gamma>0.
\)
Let
\(
M^\mathrm{ML}_m(t,N)
=\E\left[\left(\frac1N\sum_{j=1}^N L^\mathrm{ML}_j\right)^m\right],
\)
where the $L^\mathrm{ML}_j$ are i.i.d.\ copies of $L^\mathrm{ML}$. These moments follow from the individual moments above using the binomial convolutions in \cref{lem:moments}. Including the probability $p$ that a row is sampled gives
\begin{equation}\label{eq:Mp-mow}
M^\mathrm{ML}_m(t,N,p)
=\E\left[\left((1-p)+\frac pN\sum_{j=1}^N L^\mathrm{ML}_j\right)^m\right]
=\sum_{k=0}^m\binom mk(1-p)^{m-k}p^k\,M^\mathrm{ML}_k(t,N).
\end{equation}
The next lemma is the analogue of \cref{lem:decouple} for one level.
 
\begin{lemma}\label{lem:decouple-mow}
Let $\alpha>1$, let $m>\alpha$ be an integer, let $\gamma=\frac{m(\alpha-1)p}{(m-\alpha)N}$, and let $t=B/\sigma$. For all
$a,a'$ with $\|a\|_2,\|a'\|_2\le B$,
\[
\E_{P^{(N,p)}_{a'}}\Big[\Big(\frac{P^{(N,p)}_a}{P^{(N,p)}_{a'}}\Big)^{\alpha}\Big]\ \le
M^\mathrm{ML}_m(t,N,p)^{\alpha/m}\,K^\mathrm{ML}(\gamma,t)^{N(m-\alpha)/m}.
\]
\end{lemma}

Now, we can state the privacy guarantees for ~\cref{alg:hmow}.
 
\begin{theorem}[Privacy of \cref{alg:hmow}]\label{thm:hmow}
Let $t=B/\sigma$. For each level $\kappa\in[C]$ and $\alpha>1$,
define
\[
\rho_\kappa(\alpha)
=
\min_{\substack{m\in\mathbb N\\m>\alpha}}
\frac{1}{\alpha-1}
\left[
\frac{\alpha}{m}
\log M^\mathrm{ML}_m(t,N_\kappa,p_\kappa)
+
\frac{N_\kappa(m-\alpha)}{m}
\log K^\mathrm{ML}\left(
\frac{m(\alpha-1)p_\kappa}{(m-\alpha)N_\kappa},t
\right)
\right].
\]
Define
\(
\rho_{\textrm{MOW}}(\alpha)
:=
\sum_{\kappa=1}^C\rho_\kappa(\alpha).
\)
Then, for every pair of neighboring datasets and every $\alpha>1$,
the release of Algorithm~\ref{alg:hmow} with noise level $\sigma$
satisfies
\[
D_\alpha\big((\tilde A,w)\,\|\,(\tilde A',w)\big)
\leq \rho_{\textrm{MOW}}(\alpha).
\]
Consequently, under replace-one adjacency, the mechanism is
$(\varepsilon,\delta)$-differentially private, where
\[
\varepsilon
=
\min_{\alpha>1}
\left\{
\rho_{\textrm{MOW}}(\alpha)
+\log\frac{\alpha-1}{\alpha}
-\frac{\log\delta+\log\alpha}{\alpha-1}
\right\}.
\]
\end{theorem}

\begin{proof}[Proof sketch]
 The levels are independent under both datasets, so their Rényi divergence bounds add~\citep{vEH14}. By \cref{lem:decouple-mow}, each level's divergence is bounded in both directions. Joint quasi-convexity then extends the bound to the unconditional output distributions. Since the weights $w$ are data-independent, releasing them does not affect the privacy guarantee. Finally, the stated conversion to $(\varepsilon,\delta)$-DP is the same as in \cref{thm:hcs}. The full proof is given in \cref{app:hmow}.
\end{proof}
 

The utility analysis of \cref{alg:hmow} relies only on the independence of $S$ from $A$ and $E$. Since \cref{alg:hmow} releases the same sketch as \citet{MOW23}, its lopsided-embedding guarantee remains unchanged. The noise enters the objective through $E\bar\beta_{-1}$. Each entry of this vector is distributed as $\mathcal{N}(0,\sigma^2\|\bar\beta_{-1}\|_2^2)$, so the resulting perturbation to $\|W\tilde A\bar\beta_{-1}\|_1$ grows with $\|\bar\beta_{-1}\|_2$. This motivates restricting the optimization to a bounded set of coefficients, as in the guarantee below.


\begin{theorem}[Utility of \cref{alg:hmow}]\label{cor:hmowutil}
Let
\(
\beta^*=\arg\min_{\beta\in\R^d}\|A\bmo\|_1,
\)
and suppose that $C_\beta$ is a public bound satisfying
$C_\beta\geq\|\beta^*\|_\infty$. Assume that $(SA,w)$ satisfies
the lopsided embedding condition of~\cite{MOW23}. Let $\beta$
be a minimizer of $\|W\tilde A\bar\beta_{-1}\|_1$ over
$\|\bar\beta\|_\infty\leq C_\beta$, where $\tilde A$ is the
release of Algorithm~\ref{alg:hmow}. Then, for every $\tau>0$,
with additional probability at least $1-2e^{-\tau^2/2}$,
\[
\|X\beta-y\|_1
\leq
\frac{1}{1-\mu}
\left(
O(1)\,\OPT_1
+
2B\sigma^\star_{\textrm{MOW}}\|w\|_2
\big(\sqrt r+\sqrt{d+1}+\tau\big)
\sqrt{dC_\beta^2+1}
\right).
\]
\end{theorem}


Hiding the sketching randomness reduces the noise level, but it does not change the weights. In particular, the uniform level still assigns each bucket weight $n/N_u$. Consequently, the noise term in \cref{cor:hmowutil} retains the factor $\|w\|_2\ge n/\sqrt{N_u}$. 

Now, we discuss the calibration of the variance of the noise in \cref{alg:hmow}. As in \cref{eq:sigmastar}, the privacy bound $\rho_{\rm MOW}$ depends on $B$ and $\sigma$ only through $t=B/\sigma$, and on the sketch only through its public parameters $\{(N_\kappa,p_\kappa)\}$. Let $\varepsilon_{\rm MOW}(s;\delta)$ denote the bound on $\varepsilon$ obtained from \cref{thm:hmow} with $t=1/s$. We define
\begin{equation}\label{eq:sigmastar-mow}
\sigma^\star_{\rm MOW}(\varepsilon,\delta)=\inf\big\{s>0:\ \varepsilon_{\rm MOW}(s;\delta)\le\varepsilon\big\}.
\end{equation}
We compute this value using the same grid of Rényi orders and bisection procedure as for $\sigma^\star$. 

\section{Experiments}\label{sec:exp}

We evaluate on California Housing ($n=20640$, $d=8$), with all features and the response min--max scaled to $[0,1]$. We use no intercept and set $B=2$ since the largest row norm is $1.90$, so clipping is unnecessary. We set $\delta=1/n^2$. For $\ell_2$ regression, we report the mean squared error (MSE), $\|X\beta-y\|_2^2/n$, and for $\ell_1$ regression, the mean absolute error (MAE), $\|X\beta-y\|_1/n$. Results are reported as mean $\pm$ standard deviation over $10$ runs. 
All private methods satisfy $(\varepsilon,\delta)$-DP under replace-one adjacency.

We compare against various baseline methods which we enumerate here.
\begin{enumerate}
    \item \textbf{Non-private regression : } We solve the ordinary least squares (OLS) as well as least absolute deviation regression (LAD) on the full dataset, using \texttt{scipy.optimize}.
    \item \textbf{AdaSSP : } As a private baseline on the full dataset for $\ell_2$ regression, we solve the private regression problem using \texttt{AdaSSP}~\cite{wang2018revisitingdifferentiallyprivatelinear} which is a sufficient statistics perturbation based method. We implement this method using the \texttt{AutoDP} \footnote{\url{https://github.com/yuxiangw/autodp}} package.
    \item \textbf{DP-SGD : } This is another private baseline for both $\ell_1$ and $\ell_2$ regression on the full dataset, where the problem is solved using private stochastic gradient descent (SGD). We used the \texttt{Opacus} \footnote{\url{https://opacus.ai/}} library for this purpose. We achieved the lowest losses at $20$ epochs for batch size $32$, maximum gradient norm $1.0$ and learning rate of $0.1$ with the $SGD$ optimizer.
    \item \textbf{JL Sketch : } This is a private sketch baseline for $\ell_2$ regression, implementing Sheffet's JL Sketch ~\cite{She17}.
\end{enumerate}

For $\ell_2$ regression, we set $r=100$ and compare OLS, DP-SGD~\cite{Abadi_2016}, AdaSSP~\cite{wang2018revisitingdifferentiallyprivatelinear}, the JL sketch of \cite{She17}, Algorithm~\ref{alg:cs}, and Algorithm~\ref{alg:hcs}. For the JL sketch, the private test of $\sigma_{\min}(A)$ fails for every finite $\varepsilon$ on this dataset, so the method always adds ridge regularization with $\lambda_{\textrm JL}=w^2$. Algorithm~\ref{alg:cs} uses noise level $\sigma=\sigma_{\textrm AGM}(\varepsilon,\delta;2B)$~\cite{BW18}, while Algorithm~\ref{alg:hcs} uses $\sigma=B\sigma^\star$ (Equation~\eqref{eq:sigmastar}).

For $\ell_1$ regression, we set $r=800$ and compare LAD, DP-SGD, Algorithm~\ref{alg:l1}, and Algorithm~\ref{alg:hmow}. The two sketching methods use noise levels $\sigma=\sigma_{\textrm AGM}(\varepsilon,\delta;2B\cmax)$ and $\sigma=B\sigma^\star_{\textrm MOW}$ (Equation~\eqref{eq:sigmastar-mow}), respectively. Both use the sketch structure of \cite{MOW23}, with $b=2$, $s=2$, $h_m=9$ levels of $N=71$ buckets, and $N_u=160$. We solve the resulting optimization problem as a linear program subject to $\|\bar\beta\|_\infty\le C_\beta=20$. 

The $\ell_2$ and $\ell_1$ regression errors are shown in Table~\ref{tab:exp}.


\begin{table}[t]
\caption{Error on California Housing, mean $\pm$ standard deviation over $10$ runs. Left: MSE for $\ell_2$ regression with
$r=100$. Right: MAE for $\ell_1$ regression with $r=800$ and $b=2$. At $\varepsilon=\infty$ no noise is added. Bold entries mark a hidden sketch that is more accurate than its public counterpart
(Algorithm~\ref{alg:hcs} against Algorithm~\ref{alg:cs}, Algorithm~\ref{alg:hmow} against Algorithm~\ref{alg:l1}). The predictor
$\beta=0$ has MSE
$0.213$ and MAE $0.396$.}
\label{tab:exp}
\begin{center}
\setlength{\tabcolsep}{2.5pt}
\resizebox{\linewidth}{!}{%
\begin{tabular}{r|cccccc|cccc}
\toprule
 & \multicolumn{6}{c|}{$\ell_2$: MSE} & \multicolumn{4}{c}{$\ell_1$: MAE}\\
$\varepsilon$ & OLS & DP-SGD & AdaSSP & JL & Alg.~\ref{alg:cs} & Alg.~\ref{alg:hcs} & LAD & DP-SGD & Alg.~\ref{alg:l1} & Alg.~\ref{alg:hmow}\\
\midrule
0.1  & 0.025{\tiny$\pm$.00} & 0.120{\tiny$\pm$.06} & 0.076{\tiny$\pm$.01} & 0.145{\tiny$\pm$.05} & 0.187{\tiny$\pm$.07} & \textbf{0.072}{\tiny$\pm$.01} & 0.112{\tiny$\pm$.00} & 0.188{\tiny$\pm$.02} & 0.393{\tiny$\pm$.06} & 0.401{\tiny$\pm$.06}\\
0.25 & 0.025{\tiny$\pm$.00} & 0.038{\tiny$\pm$.00} & 0.055{\tiny$\pm$.00} & 0.100{\tiny$\pm$.02} & 0.198{\tiny$\pm$.08} & \textbf{0.069}{\tiny$\pm$.01} & 0.112{\tiny$\pm$.00} & 0.128{\tiny$\pm$.01} & 0.355{\tiny$\pm$.05} & \textbf{0.342}{\tiny$\pm$.05}\\
0.5  & 0.025{\tiny$\pm$.00} & 0.029{\tiny$\pm$.00} & 0.049{\tiny$\pm$.00} & 0.071{\tiny$\pm$.01} & 0.219{\tiny$\pm$.07} & \textbf{0.054}{\tiny$\pm$.01} & 0.112{\tiny$\pm$.00} & 0.120{\tiny$\pm$.00} & 0.356{\tiny$\pm$.05} & \textbf{0.300}{\tiny$\pm$.03}\\
0.75 & 0.025{\tiny$\pm$.00} & 0.028{\tiny$\pm$.00} & 0.045{\tiny$\pm$.00} & 0.069{\tiny$\pm$.02} & 0.175{\tiny$\pm$.04} & \textbf{0.045}{\tiny$\pm$.01} & 0.112{\tiny$\pm$.00} & 0.119{\tiny$\pm$.00} & 0.374{\tiny$\pm$.06} & \textbf{0.253}{\tiny$\pm$.03}\\
1.0  & 0.025{\tiny$\pm$.00} & 0.028{\tiny$\pm$.00} & 0.042{\tiny$\pm$.00} & 0.062{\tiny$\pm$.01} & 0.153{\tiny$\pm$.03} & \textbf{0.044}{\tiny$\pm$.00} & 0.112{\tiny$\pm$.00} & 0.119{\tiny$\pm$.00} & 0.359{\tiny$\pm$.06} & \textbf{0.227}{\tiny$\pm$.02}\\
2.5  & 0.025{\tiny$\pm$.00} & 0.027{\tiny$\pm$.00} & 0.036{\tiny$\pm$.00} & 0.049{\tiny$\pm$.01} & 0.080{\tiny$\pm$.02} & \textbf{0.040}{\tiny$\pm$.01} & 0.112{\tiny$\pm$.00} & 0.119{\tiny$\pm$.00} & 0.355{\tiny$\pm$.07} & \textbf{0.186}{\tiny$\pm$.01}\\
5.0  & 0.025{\tiny$\pm$.00} & 0.027{\tiny$\pm$.00} & 0.031{\tiny$\pm$.00} & 0.044{\tiny$\pm$.01} & 0.055{\tiny$\pm$.01} & \textbf{0.033}{\tiny$\pm$.00} & 0.112{\tiny$\pm$.00} & 0.118{\tiny$\pm$.00} & 0.268{\tiny$\pm$.05} & \textbf{0.186}{\tiny$\pm$.01}\\
$\infty$ & 0.025{\tiny$\pm$.00} & 0.027{\tiny$\pm$.00} & 0.025{\tiny$\pm$.00} & 0.028{\tiny$\pm$.00} & 0.027{\tiny$\pm$.00} & 0.027{\tiny$\pm$.00} & 0.112{\tiny$\pm$.00} & 0.118{\tiny$\pm$.00} & 0.115{\tiny$\pm$.00} & 0.115{\tiny$\pm$.00}\\
\bottomrule
\end{tabular}}
\end{center}
\end{table}

\section{Results and Discussion}\label{sec:results}


\paragraph{Hiding the randomness.}
The Hidden CountSketch requires a ridge parameter $\lambda_{\textrm hid}$ that is $44$ to $587$ times smaller than $\lambda_{\textrm CS}$ and $2.7$ to $6.9$ times smaller than $\lambda_{\textrm JL}$. This improvement comes from hiding the bucket assignment, which spreads each row's contribution across $r$ buckets, and the sign.  Hidden CountSketch is the most accurate sketch at every finite $\varepsilon$. At $\varepsilon=0.1$, its error is half that of the JL sketch ($0.072$ versus $0.145$). It is also the most accurate private method at this privacy level and remains within $0.014$ of AdaSSP at other privacy levels. DP-SGD performs best for $\varepsilon\ge0.25$.

Without noise, the multi-level sketch is accurate, achieving an MAE of \(0.115\), close to the LAD baseline at \(0.112\). The results change substantially when noise is added. Algorithm~\ref{alg:l1} remains close to the trivial solution \(\beta=0\), with MAE between \(0.355\) and \(0.393\) for all \(\varepsilon\le 2.5\). In contrast, Algorithm~\ref{alg:hmow} improves steadily as \(\varepsilon\) increases: its MAE drops from \(0.342\) at \(\varepsilon=0.25\) to \(0.186\) for \(\varepsilon\ge 2.5\). It also outperforms Algorithm~\ref{alg:l1} for every \(\varepsilon\ge 0.1\).

\section{Conclusion}
In this chapter, we developed techniques for releasing differentially private sketches for \(\ell_1\) and \(\ell_2\) regression. Our results show that hiding the randomness of the sketch can substantially improve utility, although the benefits depend on the sketch construction. Several questions remain open, including tighter privacy accounting, lower bounds for sketches with hidden randomness, and new \(\ell_1\) sketch designs that better control noise amplification. 

\section{Acknowledgments}
The authors would like to thank Debanuj Nayak from Boston University for the discussions during the initial stages of the work.

\bibliography{main}
\bibliographystyle{tmlr}

\appendix
\section{Appendix}

\section{\texorpdfstring{Proofs for \cref{sec:lb}}{Proofs for Section 4.3}}\label{app:lb}
 
\paragraph{Proof of \cref{lem:group}.}      Consider a sequence of neighboring datasets $A=A^{(0)},A^{(1)},\dots,A^{(k)}=A''$. Applying the DP inequality successively along this chain gives
\[
p_0\leq e^\varepsilon p_1+\delta
\leq \cdots
\leq e^{k\varepsilon}p_k
+\delta\sum_{i=0}^{k-1}e^{i\varepsilon}.
\]
 
\paragraph{Proof of \cref{lem:shift}.} Let
\[
a=\Pr[\eta\geq z/2],\qquad
c=\Pr[\eta\leq-z/2],\qquad
b=\Pr[|\eta|<z/2],
\]
so that $a+b+c=1$.

Applying the first indistinguishability inequality to
$\mathcal T=[z/2,\infty)$ gives
\[
\Pr[\eta\geq-z/2]
\leq e^{\varepsilon'}\Pr[\eta\geq z/2]+\delta',
\]
which implies
\[
a+b\leq e^{\varepsilon'}a+\delta'.
\]
Similarly, applying the second inequality to
$\mathcal T=(-\infty,z/2]$ gives
\[
\Pr[\eta\leq z/2]
\leq e^{\varepsilon'}\Pr[\eta\leq-z/2]+\delta',
\]
and hence
\[
b+c\leq e^{\varepsilon'}c+\delta'.
\]
Adding these inequalities and using $a+b+c=1$, we obtain
\[
1+b\leq e^{\varepsilon'}(1-b)+2\delta',
\]
which yields
\[
\Pr[|\eta|\geq z/2]=1-b
\geq\frac{2(1-\delta')}{e^{\varepsilon'}+1}.
\]
The second-moment bound follows immediately:
\[
\E[\eta^2]
\geq \frac{z^2}{4}\Pr[|\eta|\geq z/2]
\geq \frac{z^2}{2}\cdot
\frac{1-\delta'}{e^{\varepsilon'}+1}.
\]

Finally, translating both random variables by the same constant
does not change their indistinguishability. Thus, the same argument
applies to $\eta-c$ for any constant $c$.
 
\paragraph{Proof of \cref{thm:lb2}.} Since the mechanism is differentially private for every fixed $S$,
we may fix $S$ and choose the input datasets accordingly. Fix an
output row $j$, a unit vector $u$, and a set $G$ of $k$ columns
attaining $s_j(k)$. Construct $A$ and $A''$ as follows. For each
$i\in G$, set
\[
A_{i,:}=-\operatorname{sign}(S_{ji})Bu^\top,
\qquad
A''_{i,:}=\operatorname{sign}(S_{ji})Bu^\top,
\]
and let all other rows be identical in the two datasets, with
norm at most $B$. By construction,
\[
(SA''-SA)_{j,:}u
=2B\sum_{i\in G}|S_{ji}|
=2Bs_j(k)=:z.
\]

Let $\eta=E_{j,:}u$. Since $E$ is independent of the data
conditioned on $S$, $\eta$ has the same distribution under both
datasets. Moreover, the map $\tilde A\mapsto\tilde A_{j,:}u$
is a post-processing of the mechanism's output. Thus, by
Lemma~\ref{lem:group}, applied in both directions, the random
variables $c+\eta$ and $c+z+\eta$ are
$(k\varepsilon,\delta_k)$-indistinguishable, where
$c=(SA)_{j,:}u$. Translating both random variables by $-c$,
we can apply Lemma~\ref{lem:shift} to $\eta$ and $\eta+z$.
It follows that
\[
\E[\eta^2]
\geq
\frac{z^2}{2}\cdot
\frac{1-\delta_k}{e^{k\varepsilon}+1}
=
\frac{2B^2s_j(k)^2(1-\delta_k)}
{e^{k\varepsilon}+1}.
\]
This proves the first inequality.

For the matrix bound, let $u$ be any unit vector in
$\R^{d+1}$. Summing the first inequality over all output rows gives
\[
\begin{aligned}
u^\top\E[E^\top E]u
&=\sum_{j=1}^r\E[(E_{j,:}u)^2]\\
&\geq
\frac{2B^2(1-\delta_k)}{e^{k\varepsilon}+1}
\sum_{j=1}^r s_j(k)^2.
\end{aligned}
\]
Since this holds for every unit vector $u$, the claimed
positive semidefinite matrix inequality follows.
 
\paragraph{Proof of \cref{cor:lbcs}.} Suppose $S$ is a CountSketch in which every bucket receives at least
$k$ rows of $A$. Then
\[
\E[E^\top E]\succeq\lambda_{\textrm{pub}}I_{d+1},
\qquad
\lambda_{\textrm{pub}}
\geq
\frac{2rB^2k^2(1-\delta_k)}{e^{k\varepsilon}+1}.
\]
In particular, choosing $k=\lceil 2/\varepsilon\rceil$ gives
\[
\lambda_{\textrm{pub}}
\geq
\frac{8rB^2(1-\delta_k)}
{\varepsilon^2(e^{2+\varepsilon}+1)}.
\]
 
\paragraph{Proof of \cref{thm:lb1}.} Let
\[
\tilde A_{j,:}\bmo=c_j+\eta_j,
\qquad
c_j=(SA\bmo)_j,
\qquad
\eta_j=E_{j,:}\bmo.
\]
For any constant $c$, we have
\[
\E|c+\eta_j|
\geq
\min_m\E|\eta_j-m|
=
\E|\eta_j-\operatorname{med}(\eta_j)|,
\]
where $\operatorname{med}(\eta_j)$ denotes a median of $\eta_j$.

Fix an output row $j$, and choose an input row $i$ attaining
$|S_{ji}|=\|S_{j,:}\|_\infty$. Consider two neighboring datasets
that differ only in row $i$, which is set to
\[
-\operatorname{sign}(S_{ji})
\frac{B\bmo^\top}{\|\bmo\|_2}
\quad\text{and}\quad
+\operatorname{sign}(S_{ji})
\frac{B\bmo^\top}{\|\bmo\|_2},
\]
respectively. All other rows are identical. The corresponding
values of $\tilde A_{j,:}\bmo$ differ by
\[
z=2B\|\bmo\|_2\|S_{j,:}\|_\infty.
\]

As in the proof of Theorem~\ref{thm:lb2}, the distribution of
$\eta_j$ is the same under both datasets. By post-processing and
Lemma~\ref{lem:shift}, applied with $k=1$ to
$\eta_j-\operatorname{med}(\eta_j)$, we obtain
\[
\Pr\left[
|\eta_j-\operatorname{med}(\eta_j)|\geq z/2
\right]
\geq
\frac{2(1-\delta)}{e^\varepsilon+1}.
\]
Consequently,
\[
\E|\eta_j-\operatorname{med}(\eta_j)|
\geq
\frac{z(1-\delta)}{e^\varepsilon+1}
=
\frac{2B\|\bmo\|_2\|S_{j,:}\|_\infty(1-\delta)}
{e^\varepsilon+1}.
\]
Combining this with the initial bound, multiplying by $w_j$,
and summing over all output rows $j$ proves the theorem.
 
\paragraph{Proof of \cref{cor:lbl1}.} For a $\{0,1\}$ sketch, $\|S_{j,:}\|_\infty$ equals $1$ if
bucket $j$ is non-empty and $0$ otherwise. Thus, the first
inequality follows directly from Theorem~\ref{thm:lb1}.

At the uniform level, each input row is retained with probability
$N_u/n$ and, if retained, is assigned uniformly to one of the
$N_u$ buckets. Hence, each row lands in a given bucket with
probability $1/n$, independently of the other rows. It follows
that each bucket is non-empty with probability
\[
1-\left(1-\frac1n\right)^n\geq1-e^{-1}.
\]
Since every bucket has weight $n/N_u$, linearity of expectation
gives
\[
\E_S W_{\neq\emptyset}
\geq
N_u\cdot\frac{n}{N_u}(1-e^{-1})
=(1-e^{-1})n.
\]
Taking expectation over $S$ in the first inequality yields the
claimed lower bound.
 
\section{\texorpdfstring{Proofs for \cref{sec:hcs}}{Proofs for Section 5}}\label{app:hcs}
 
\paragraph{Proof of \cref{lem:lr}.} The density of $\mathcal{N}\big(g\,e_j a^\top,\sigma^2 I_{r(d+1)}\big)$, divided by that of $R$, evaluated at $X$, is
\[
\exp\left(-\frac{\|X-ge_ja^\top\|_F^2-\|X\|_F^2}{2\sigma^2}\right).
\]
Since the two matrices differ only in row $j$ and $g^2=1$, we have
\[
\|X-ge_ja^\top\|_F^2
=\|X\|_F^2-2g\langle a,X_j\rangle+\|a\|_2^2.
\]
Substituting this expression into the density ratio gives
\(
\exp\left(\frac{g\langle a,X_j\rangle}{\sigma^2}
-\frac{\|a\|_2^2}{2\sigma^2}\right).
\)
Averaging over the two possible signs, $g\in\{-1,+1\}$, gives
\[
\ell_a(X_j)
=e^{-t_a^2/2}\cosh\left(\frac{\langle a,X_j\rangle}{\sigma^2}\right),
\]
where $\cosh(x)=(e^x+e^{-x})/2$. Averaging over the $r$ possible buckets then gives the stated likelihood ratio.

Under $R$, the rows $X_1,\ldots,X_N$ are independent, with
$X_j\sim\mathcal N(0,\sigma^2I_{d+1})$. Thus,
\[
\frac{\langle a,X_j\rangle}{\sigma^2}
\sim\mathcal N(0,t_a^2),
\qquad t_a=\frac{\|a\|_2}{\sigma},
\]
and the likelihood-ratio terms $\ell_a(X_j)$ are independent and identically distributed.

 \begin{proposition}[Second-order privacy loss]\label{prop:chi2}
Let $t=t_a$. Then
\(
\E_R\left[\left(\frac{P_a}{R}\right)^2\right]
=1+\frac{\cosh(t^2)-1}{r}.
\)
If both $h$ and $g$ are public, this expectation is $e^{t^2}$. If only $h$ is hidden, it is $1+(e^{t^2}-1)/r$, and if only $g$ is hidden, it is $\cosh(t^2)$.
\end{proposition}

\begin{proof} Using the fact that $\E \exp(sZ) = \exp(s^2/2),$ we get $$\E\cosh(sZ)=e^{s^2/2}, \quad \text { and }\E L=1.$$ Also, using $\cosh^2x=(1+\cosh2x)/2$, we get
$$\E L^2=e^{-t^2}(1+e^{2t^2})/2=\cosh(t^2).$$ 
Then, expanding the square of the average, we have
$$\E(\frac1r\sum_jL_j)^2=\frac{1}{r^2}\big(r\E L^2+r(r-1)(\E L)^2\big)=1+(\cosh(t^2)-1)/r.$$

If $h$ and $g$ are public, the ratio $\frac{P_a}{R} \left(X\right) = \ell_a(X) = e^{tZ-t^2/2}$, since $j=1$ (the adversary knows  where $a$ is mapped to), and $\E_R\left[\left(\frac{P_a}{R}\right)^2\right] = e^{t^2}$. 

If only $h$ is secret, the $\cosh$ factor does not appear and each term in the ratio has mean $1$ and second moment $e^{t^2}.$ Using the same expansion gives $\E_R\left[\left(\frac{P_a}{R}\right)^2\right] = 1+(e^{t^2}-1)/r$. 

If only $g$ is secret, it is $\cosh(t^2)$.
\end{proof}

\begin{lemma}[Exact moments]\label{lem:moments}
For every integer $c\ge 0$,
\[
\E L^c=e^{-ct^2/2}\,2^{-c}\sum_{k=0}^{c}\binom ck
e^{(2k-c)^2t^2/2}.
\]
For independent random variables $U$ and $V$, the moments of their sum satisfy
\[
\E[(U+V)^m]=\sum_{i=0}^m\binom mi\,\E[U^i]\,\E[V^{m-i}].
\]
Applying this identity repeatedly gives $M_m(t,r)$ from the moments of $L$ using $O(\log r)$ binomial convolutions, with all terms nonnegative.
\end{lemma}
  \begin{proof}
   By the binomial expansion we get
$\cosh^c(tZ)=2^{-c}\sum_{k=0}^c\binom cke^{(2k-c)tZ}.$ Also, $\E e^{(2k-c)tZ}=e^{(2k-c)^2t^2/2}$. 

The convolution identity
follows from the binomial expansion of $(U+V)^m$ and  the independence of $U$ and $V$. The moments of a sum of $r$ independent copies are
obtained by repeated doubling along the binary expansion of $r$, and $M_m(t,r)$ is the $m$th moment of the sum divided by
$r^m$.
\end{proof}
 
\paragraph{Proof of \cref{lem:decouple}.} Since $R$ has a positive density, 
$$\E_{P_{a'}}[(P_a/P_{a'})^\alpha]=\int P_a^\alpha P_{a'}^{1-\alpha}=\E_R[U^\alpha V^{1-\alpha}]$$ with $U=P_a/R$ and $V=P_{a'}/R$.
Applying H\"older's inequality with $p=m/\alpha$ and $q=m/(m-\alpha)$ we get
$$\E_R[U^\alpha V^{1-\alpha}]\le(\E_RU^m)^{\alpha/m}(\E_RV^{-q(\alpha-1)})^{1/q},$$ and $\E_RU^m=M_m(t_a,r)$ by Lemma~\ref{lem:lr}.

By  arithmetic mean - geometric mean inequality, $$V=\frac1r\sum_j\ell_{a'}(X_j)\ge\prod_j\ell_{a'}(X_j)^{1/r},$$ so
$V^{-q(\alpha-1)}\le\prod_j\ell_{a'}(X_j)^{-\gamma}$ with $\gamma=q(\alpha-1)/r$. Since the rows are independent under $R$, 
$\E_RV^{-q(\alpha-1)}\le K(\gamma,t_{a'})^r$, and raising to the power $1/q=(m-\alpha)/m$ gives the bound at $(t_a,t_{a'})$.
 
The bound just obtained is evaluated at the actual lengths $t_a=\|a\|_2/\sigma$ and $t_{a'}=\|a'\|_2/\sigma$, whereas the
accountant uses the single value $t=B/\sigma$ for every pair of neighbours. This is valid if $M_m(t_a,r)$ is
non-decreasing in $\|a\|_2$ and $K(\gamma,t_{a'})$ is non-decreasing in $\|a'\|_2$, that is, if a longer row is never
harder to detect. Since both of the factors depend on $a$ or $a'$ through $t_a$ or $t_{a'}$ (as $\langle a,X_j\rangle/\sigma^2\sim N(0,t_a^2)$ or $\langle a',X_j\rangle/\sigma^2\sim N(0,t_{a'}^2)$ under $R$) whatever the direction of $a$ or $a'$ and $t_a$ or $t_{a'}$ depend  only on $\Vert a\Vert_2$ or $\Vert a' \Vert_2$, so it suffices to compare a row $a$ with $ca$ for $c\in[0,1]$.

Let $\xi$ be a  matrix of i.i.d standard Gaussians and consider the transformation
$X\mapsto cX+\sigma\sqrt{1-c^2}\,\xi$. Multiplying by $c$ shrinks the row to $ca$ and the noise variance to $c^2\sigma^2$,
and the added term restores the variance to $\sigma^2$. Hence this transformation maps $R$ to $R$ and $P_a$ to $P_{ca}$. By the
data-processing inequality for R\'enyi divergence~\cite{vEH14}, applying the same transformation to both distributions cannot
increase their divergence, so $D_m(P_{ca}\|R)\le D_m(P_a\|R)$.

Now, we use this to show that the factors are non-decreasing in the length of the rows. To do this, we write both of the factors as divergences. From the definition of $D_m$,
$M_m(t_a,r)=\E_R[(P_a/R)^m]=e^{(m-1)D_m(P_a\|R)}$, which is therefore non-decreasing in $\|a\|_2$. Similarly, 
$K(\gamma,t_{a'})=\E_{R}[(R/P_{a'})^{\gamma}]=e^{\gamma D_{1+\gamma}(R\|P_{a'})}$, which is non-decreasing in $\|a'\|_2$.
Both factors are therefore largest at $\|a\|_2=\|a'\|_2=B$, that is, at $t_a=t_{a'}=t$.

\paragraph{Proof of \cref{thm:hcs}.} By the reduction above, the distribution of $\tilde A$ under each data set is a mixture of translations of $P_a$ (or $P_{a'}$), with common mixing weights determined by the hashes and signs of the remaining rows. Since Rényi divergence is invariant under a common translation and jointly quasi-convex~\cite{vEH14}, we obtain
\[
D_\alpha(\tilde A\|\tilde A')
\leq \sup_{a,a'} D_\alpha(P_a\|P_{a'})
\leq \rho(\alpha),
\]
where the last inequality follows from Lemma~\ref{lem:decouple}. We may take the minimum over $m$, since the lemma holds for every admissible choice of $m$. Moreover, the bound in Lemma~\ref{lem:decouple} is independent of the order of $a$ and $a'$, so the same argument gives
\[
D_\alpha(\tilde A'\|\tilde A)\leq \rho(\alpha).
\]
Finally, applying the conversion in~\cite{CKS20}, we obtain $(\varepsilon,\delta)$-DP with
\[
\varepsilon
=
\rho(\alpha)
+\log\frac{\alpha-1}{\alpha}
-\frac{\log\delta+\log\alpha}{\alpha-1}.
\]
We choose $\alpha$ to minimize this bound.

\paragraph{Proof of \cref{cor:hcsutil}}\label{app:util}
Let $U\in\R^{(n+d+1)\times(d+1)}$ be an orthonormal basis for the column space of the augmented matrix $\hat A$, written as $U=[U_1;U_2]$, where $U_1\in\R^{n\times(d+1)}$ and $U_2 \in \R^{(d+1) \times  (d+1)}$. Define $f_\ell=\|U_\ell\|_F^2$, so that $f_1+f_2=d+1$. We can write
\[
U^\top T^\top TU-I=V+W+C+C^\top,
\]
where
\[
\begin{aligned}
V&=U_1^\top S^\top SU_1-U_1^\top U_1,\\
W&=U_2^\top G^\top GU_2-U_2^\top U_2,\\
C&=U_1^\top S^\top GU_2.
\end{aligned}
\]

Since $S$ and $G$ are independent and $\E V=\E W=0$, we have $\E\langle V,W\rangle=0$. Moreover, $C$ is odd in $G$, whereas $V$ and $W$ are even in $G$, so $\E\langle V,C\rangle=\E\langle W,C\rangle=0$. For CountSketch, we have
\[
\E\|M^\top S^\top SM-M^\top M\|_F^2
\leq\frac{2}{r}\|M\|_F^4
\]
for every $M$~\cite{CW17}. The same bound holds for $G$, whose entries are $\N(0,1/r)$. Furthermore,
\[
\E\|C\|_F^2
=\frac{1}{r}\E\|SU_1\|_F^2\,f_2
=\frac{f_1f_2}{r}.
\]
Since $\|C+C^\top\|_F^2\leq4\|C\|_F^2$, we obtain
\[
\begin{aligned}
\E\|U^\top T^\top TU-I\|_F^2
&\leq\frac{2}{r}(f_1^2+f_2^2)+\frac{4}{r}f_1f_2\\
&=\frac{2(d+1)^2}{r}.
\end{aligned}
\]
By Markov's inequality and $\|\cdot\|_2\leq\|\cdot\|_F$, with probability at least $1-\nu$,
\[
\|U^\top T^\top TU-I\|_2
\leq(d+1)\sqrt{\frac{2}{r\nu}}
\leq\mu.
\]
This establishes the subspace embedding.

Now, let $\tilde v=[\tilde\beta;-1]$. By the embedding and the optimality of $\beta$,
\[
(1-\mu)\|\hat A\bmo\|_2^2
\leq\|T\hat A\bmo\|_2^2
\leq\|T\hat A\tilde v\|_2^2
\leq(1+\mu)\|\hat A\tilde v\|_2^2.
\]
Using
\[
\|\hat A\bmo\|_2^2
\geq\|X\beta-y\|_2^2+\lambda_{\mathrm{hid}}
\]
and
\[
\|\hat A\tilde v\|_2^2
=\|X\tilde\beta-y\|_2^2
+\lambda\|\tilde\beta\|_2^2+\lambda_{\mathrm{hid}},
\]
rearranging and minimizing over $\tilde\beta$ gives the claim.


\section{\texorpdfstring{Proofs for \cref{sec:hcs-l1}}{Proofs for Section 5.3}}\label{app:hmow}

\paragraph{Proof of \cref{lem:lr-mow}.}
The density of $\mathcal N(e_ja^\top,\sigma^2I_{N(d+1)})$ relative to $R$ at $X$ is
\(
\exp\left(-\frac{\|X-e_ja^\top\|_F^2-\|X\|_F^2}{2\sigma^2}\right).
\)
Since $e_ja^\top$ is zero in every row except row $j$, we have
\[
\|X-e_ja^\top\|_F^2
=\|X\|_F^2-2\langle a,X_j\rangle+\|a\|_2^2.
\]
Substituting this identity into the density ratio gives
\[
\exp\left(\frac{\langle a,X_j\rangle}{\sigma^2}
-\frac{\|a\|_2^2}{2\sigma^2}\right)
=\ell_a^{\mathrm{ML}}(X_j).
\]
The  likelihood ratio with respect to $R$ is therefore
\[
\frac{P_a^{(N,p)}}{R}(X)
=(1-p)+\frac{p}{N}\sum_{j=1}^N
\ell_a^{\mathrm{ML}}(X_j),
\]
Since the rows $X_1,\ldots,X_N$ are independent and identically distributed as $\mathcal N(0,\sigma^2I_{d+1})$, we have
\(
\frac{\langle a,X_j\rangle}{\sigma^2}
\sim\mathcal N(0,t_a^2),
\quad \text{ where } t_a=\frac{\|a\|_2}{\sigma}.
\)
Thus, it follows that each $\ell_a^{\mathrm{ML}}(X_j)$ has the same distribution as
\[
L^{\mathrm{ML}}=\exp(t_aZ-t_a^2/2),
\qquad Z\sim\mathcal N(0,1).
\]
Independence of the rows also gives independence of these likelihood-ratio terms.

\paragraph{Moments and the second moment of one level.}
Since $\E e^{sZ}=e^{s^2/2}$, $\E(L^\mathrm{ML})^c=e^{-ct^2/2}e^{c^2t^2/2}=e^{c(c-1)t^2/2}$ and
$\E(L^\mathrm{ML})^{-\gamma}=e^{\gamma t^2/2}e^{\gamma^2t^2/2}=e^{\gamma(\gamma+1)t^2/2}$. The second expression in \cref{eq:Mp-mow}
follows from the binomial theorem applied to $(1-p)+pV$ with $V=\frac1N\sum_jL^\mathrm{ML}_j$, since $\E V^k=M^\mathrm{ML}_k(t,N)$. For the
second moment, $\E V=1$ and $\E V^2=\frac1N\E(L^\mathrm{ML})^2+\frac{N-1}{N}=1+(e^{t^2}-1)/N$, so
$\E[((1-p)+pV)^2]=(1-p)^2+2p(1-p)+p^2\E V^2=1+p^2(e^{t^2}-1)/N$.

\paragraph{Proof of \cref{lem:decouple-mow}.}
Let $U=P^{(N,p)}_a/R$ and $V=P^{(N,p)}_{a'}/R$. Since $R$ has a
positive density,
\[
\E_{P^{(N,p)}_{a'}}\Big[\Big(\frac{P^{(N,p)}_a}{P^{(N,p)}_{a'}}\Big)^{\alpha}\Big]=\E_R\big[U^\alpha V^{1-\alpha}\big]
\le\big(\E_RU^m\big)^{\alpha/m}\big(\E_RV^{-q(\alpha-1)}\big)^{1/q},
\]
by H\"older's inequality with exponents $m/\alpha$ and $q=m/(m-\alpha)$. By Lemma~\ref{lem:lr-mow} and Equation~\ref{eq:Mp-mow},
$\E_RU^m=M^\mathrm{ML}_m(t_a,N,p)$. 

By Lemma~\ref{lem:lr-mow}, $V$ is a weighted arithmetic mean of $1,\ell^\mathrm{ML}_{a'}(X_1),\dots,
\ell^\mathrm{ML}_{a'}(X_N)$ with weights $1-p,p/N,\dots,p/N$, so the weighted inequality of arithmetic and geometric means gives
$V\ge\prod_j\ell^\mathrm{ML}_{a'}(X_j)^{p/N}$ and
\[
V^{-q(\alpha-1)}\le\prod_{j=1}^N\ell^\mathrm{ML}_{a'}(X_j)^{-\gamma},\qquad \text{ where } \gamma=\frac{q(\alpha-1)p}{N}=\frac{m(\alpha-1)p}{(m-\alpha)N}.
\]
The rows are independent under $R$, so $\E_RV^{-q(\alpha-1)}\le K^\mathrm{ML}(\gamma,t_{a'})^N$, and raising to the power
$1/q=(m-\alpha)/m$ gives the bound at $(t_a,t_{a'}).$ .

It remains to replace $t_a$ and $t_{a'}$ by $t$. For $K^\mathrm{ML}$ this is immediate, since its closed form
$e^{\gamma(\gamma+1)t_{a'}^2/2}$ increases with $\|a'\|_2$. Fix $c\in[0,1]$ and let $\xi$ be an independent standard Gaussian matrix. The map
$X\mapsto cX+\sigma\sqrt{1-c^2}\,\xi$ keeps the noise variance at $\sigma^2$, so it sends $R$ to $R$ and
$\N(e_ja^\top,\sigma^2I)$ to $\N(ce_ja^\top,\sigma^2I)$. It acts on each component of a mixture separately, so it also sends
$P^{(N,p)}_a$ to $P^{(N,p)}_{ca}$. The map does not depend on $a$, so the data-processing inequality~\cite{vEH14} gives
$D_m(P^{(N,p)}_{ca}\|R)\le D_m(P^{(N,p)}_a\|R)$. Since $M^\mathrm{ML}_m(t_a,N,p)=e^{(m-1)D_m(P^{(N,p)}_a\|R)}$, it is
non-decreasing in $\|a\|_2$. Both factors are therefore largest at $\|a\|_2=\|a'\|_2=B$.


\paragraph{Proof of \cref{thm:hmow}.}
The outputs at different levels are conditionally independent, since sampling, hashing, and noise are independent across levels. In particular, at level $\kappa$, the output has distribution $P^{(N_\kappa,p_\kappa)}_a$ under one data set and $P^{(N_\kappa,p_\kappa)}_{a'}$ under the other. Since Rényi divergence is additive over product distributions~\cite{vEH14}, Lemma~\ref{lem:decouple-mow} gives a bound of
\(
\sum_\kappa \rho_\kappa(\alpha)
=\rho_{\textrm{MOW}}(\alpha)
\)
on the divergence of the conditional outputs, where the minimization over $m$ is performed separately for each level. This bound is independent of the order of $a$ and $a'$, so it holds in both directions.

Finally, the unconditional output distribution under each data set is a mixture, with common mixing weights, of translates of the corresponding conditional product distributions. By translation invariance and joint quasi-convexity of Rényi divergence~\cite{vEH14}, the same bound holds for the unconditional output distributions.

Since the weights $w$ are independent of the data, they do not affect the privacy analysis. Applying the conversion in~\cite{CKS20}, as in Theorem~\ref{thm:hcs}, gives the claimed $(\varepsilon,\delta)$-DP guarantee.

\paragraph{Proof of \cref{cor:hmowutil}.}
Let \(v=\beta_{-1}\) denote the augmented vector corresponding to the private solution, and let \(v^*=\beta^*_{-1}\) denote the one corresponding to the non-private optimum. By definition,
\[
\|Av\|_1=\|X\beta-y\|_1
\qquad\text{and}\qquad
\|Av^*\|_1=\OPT_1.
\]

By \cite[Thm.~1]{MOW23}, there is an event \(\mathcal E_S\), depending only on the sampling decisions and hash functions, on which
\begin{align}
\|WSA\bar\beta_{-1}\|_1
&\ge (1-\mu)\|A\bar\beta_{-1}\|_1
\quad\text{for all }\bar\beta
\text{ with }\|\bar\beta\|_\infty\le C_\beta,
\label{eq:contraction}\\
\|WSAv^*\|_1
&\le c_{\textrm dil}\,\OPT_1,
\label{eq:dilation}
\end{align}
where \(c_{\textrm dil}=O(1)\). The event has the same probability as in the public-sketch setting: Algorithm~\ref{alg:hmow} uses the same sketch and weights as \cite{MOW23}; the addition of Gaussian noise does not affect the sampling decisions or hash functions.


For any \(u\in\mathbb R^{d+1}\), Cauchy--Schwarz gives
\[
\|WEu\|_1
=\sum_j w_j |(Eu)_j|
\le \|w\|_2\|Eu\|_2
\le \|w\|_2\|E\|_2\|u\|_2.
\]
Since \(E/\sigma\) has independent standard Gaussian entries, \cite{Ver12} implies that, with probability at least \(1-2e^{-\tau^2/2}\),
\[
\|E\|_2
\le \sigma\big(\sqrt r+\sqrt{d+1}+\tau\big).
\]
Let \(\mathcal E_E\) denote this event, and define
\[
\Lambda
=\sigma\|w\|_2\big(\sqrt r+\sqrt{d+1}+\tau\big).
\]
It follows that, on \(\mathcal E_E\),
\begin{equation}\label{eq:noisebound}
\|WEu\|_1\le\Lambda\|u\|_2
\qquad\text{for all }u\in\mathbb R^{d+1}.
\end{equation}
Moreover, since \(E\) is independent of \(S\), the conditional probability of \(\mathcal E_E\), given any realization of \(S\), is at least \(1-2e^{-\tau^2/2}\).


For any feasible \(\bar\beta\), we have
\[
\|\bar\beta_{-1}\|_2^2
=\|\bar\beta\|_2^2+1
\le dC_\beta^2+1.
\]
Since both \(\beta\) and \(\beta^*\) are feasible, this gives
\begin{equation}\label{eq:normbound}
\|v\|_2,\ \|v^*\|_2
\le \sqrt{dC_\beta^2+1}.
\end{equation}


On the event \(\mathcal E_S\cap\mathcal E_E\), the contraction property, the optimality of \(\beta\) for the noisy objective, and the triangle inequality yield
\begin{align*}
(1-\mu)\|Av\|_1
&\le \|WSAv\|_1\\
&\le \|W\tilde Av\|_1+\|WEv\|_1\\
&\le \|W\tilde Av^*\|_1+\|WEv\|_1\\
&\le \|WSAv^*\|_1+\|WEv^*\|_1+\|WEv\|_1\\
&\le c_{\textrm dil}\,\OPT_1
+\Lambda\big(\|v^*\|_2+\|v\|_2\big)\\
&\le c_{\textrm dil}\,\OPT_1
+2\Lambda\sqrt{dC_\beta^2+1}.
\end{align*}
The third line follows because \(\beta\) minimizes the noisy objective over the feasible set and \(\beta^*\) belongs to that set. The fifth line uses the dilation property and the noise bound~\eqref{eq:noisebound}; the last line follows from~\eqref{eq:normbound}.

Dividing by \(1-\mu\), using \(\|Av\|_1=\|X\beta-y\|_1\), and substituting \(\sigma=B\sigma^\star_{\textrm MOW}\), we obtain
\[
\|X\beta-y\|_1
\le \frac{1}{1-\mu}\Big(
c_{\textrm dil}\,\OPT_1
+
2B\sigma^\star_{\textrm MOW}\|w\|_2
\big(\sqrt r+\sqrt{d+1}+\tau\big)
\sqrt{dC_\beta^2+1}
\Big).
\]


The utility bound holds on \(\mathcal E_S\cap\mathcal E_E\). Since \(E\) is independent of \(S\), and \(\mathcal E_S\) depends only on \(S\), we have
\[
\begin{aligned}
\Pr[\mathcal E_S\cap\mathcal E_E]
&=\mathbb E\!\left[
\mathbf 1_{\mathcal E_S}\Pr[\mathcal E_E\mid S]
\right]\\
&\ge \Pr[\mathcal E_S]\big(1-2e^{-\tau^2/2}\big).
\end{aligned}
\]
Consequently, conditional on the embedding event \(\mathcal E_S\), the utility bound holds with probability at least \(1-2e^{-\tau^2/2}\). Thus, the utility guarantee for the public sketch carries over to Algorithm~\ref{alg:hmow}, with the Gaussian noise calibrated to \(\sigma=B\sigma^\star_{\textrm MOW}\).


\end{document}